\documentclass[sigconf]{acmart}

\AtBeginDocument{%
  \providecommand\BibTeX{{%
    \normalfont B\kern-0.5em{\scshape i\kern-0.25em b}\kern-0.8em\TeX}}}

\setcopyright{none}
\renewcommand\footnotetextcopyrightpermission[1]{} 
\makeatletter
\renewcommand\@formatdoi[1]{\ignorespaces}
\makeatother

\acmConference[]{}{}
\acmBooktitle{}
\acmPrice{}
\acmISBN{}

\usepackage{amssymb}
\usepackage{amsmath}
\usepackage{blkarray}
\usepackage{color}
\usepackage{times}
\usepackage{enumitem}
\usepackage{arydshln}
\usepackage{url}
\usepackage{algorithm}
\usepackage{algorithmic}

\usepackage{flushend}
\usepackage{subcaption}
\usepackage{balance}

\newcommand{\eat}[1]{}

\newtheorem{theorem}{Theorem} 

\newtheorem{definition}{Definition}
\newtheorem{example}{Example}

\usepackage{xspace}
\newcommand{\ar}{{$\mathcal{F}^*$}\xspace}
\newcommand{\fspace}{{$\mathcal{F}$}\xspace}
\newcommand{\up}{{\sc up}\xspace}

\usepackage{graphicx}
\usepackage{comment}
\usepackage[colorinlistoftodos]{todonotes}

\begin{document}

\title{Responsible Scoring Mechanisms Through Function Sampling}

\author{Abolfazl Asudeh}
\email{asudeh@uic.edu}
\affiliation{%
 University of Illinois at Chicago
}

\author{H. V. Jagadish}
\email{jag@umich.edu}
\affiliation{%
 University of Michigan
}

\renewcommand{\shortauthors}{ }

\begin{abstract}
Human decision-makers often receive assistance from data-driven algorithmic systems that provide a score for evaluating objects, including individuals. 
The scores are generated by a function (mechanism) that takes a set of features as input and generates a score.
The scoring functions are either machine-learned or human-designed and can be used for different decision purposes such as ranking or classification.

Given the potential impact of these scoring mechanisms on individuals' lives and on society, it is important to make sure these scores are computed responsibly. 
Hence we need tools for responsible scoring mechanism design. 
In this paper, focusing on linear scoring functions, we highlight the importance of unbiased function sampling and perturbation in the function space for devising such tools.
We propose unbiased samplers for the entire function space, as well as a $\theta$-vicinity around a given function
We then illustrate the value of these samplers for designing effective algorithms in three diverse problem scenarios in the context of ranking.
Finally, as a fundamental method for designing responsible scoring mechanisms, 
we propose a novel approach for approximating the construction of the arrangement of hyperplanes.
Despite the exponential complexity of an arrangement in the number of dimensions, using function sampling, our algorithm is linear in the number of samples and hyperplanes, and independent of the number of dimensions.
\end{abstract}
\maketitle
\section{Introduction}\label{sec:intro}
Data-driven decision making is increasingly used in recent years, with significant impacts in many aspects of society. In many systems, data are reduced to a single numeric score, which is then used for the decision. For example, recidivism prediction software may consider multiple parameters to determine a numeric score that indicates the likelihood a person will re-offend. Similarly, employee prospect selection software may score and then rank applicants based on various attributes.

While data-driven decisions offer the promise of being uniform and objective, they can suffer from many imperfections. For example, a recidivism predictor may consistently rate African Americans more likely to re-offend than members of other races. In other words, it may have disparate impact, Our motivation in this paper is to assist in the development of responsible scoring systems that can avoid such harms.

\begin{figure}[!t]
    \centering
    \includegraphics[width=0.4\textwidth]{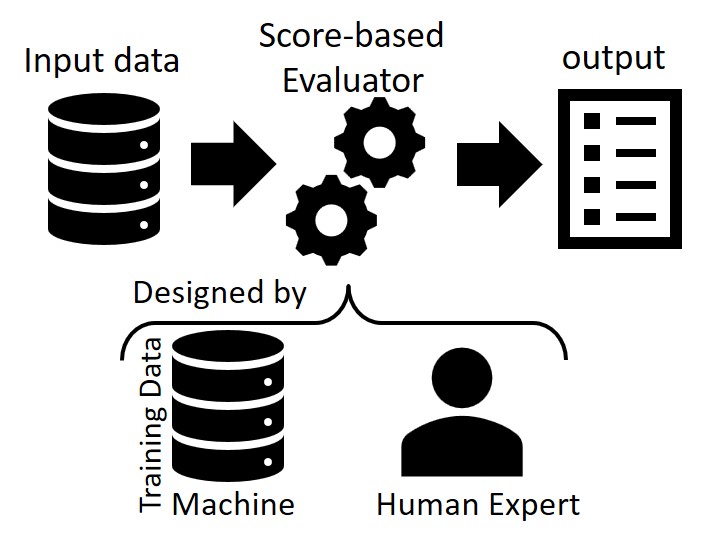}
    \caption{The general architecture of a score-based system}
    \label{fig:arch}
\end{figure}

Unfortunately, there isn't just one thing we must be mindful of to be responsible, even for fairness alone, there are numerous definitions, not all of which can be satisfied simultaneously~\cite{narayanan2018translation}.
Even though we may be interested in many desirable properties, including diversity, representation, and so on, without loss of generality we will refer to all such desirable characteristics as {\em fairness} characteristics.
In spite of their diversity, most such fairness characteristics can be defined either as a Boolean or as a score, and this is all we will rely on in this paper.

Scoring functions too, can be defined in many different ways. However, in practice, many scoring functions are linear, determined as a weighted sum of input values of features/attributes. 
Even when a scoring function is not linear, it can often be represented as a linear function after applying non-linear transformations on the attributes as a pre-processing step. For example, a multiplicative combination can be rendered linear by considering logarithms.
Therefore, in this paper, we 
concentrate on linear scoring functions.

When a machine is used to learn a fair scoring function, it is sometimes possible to express the desired fairness criteria as constraints on the learning (or optimization) problem.  There is a body of excellent recent work in this direction.
Nevertheless, there remain many situations for which such a problem statement is not possible, In such cases, we may need to explore the space of scoring functions to find one that is fair. This exploration could be a broad exploration of the entire function space, or a targeted exploration of function space close to a scoring function learned through a traditional method, without any fairness constraints.

Unfortunately, the space of possible functions is very large, even when we restrict our attention to linear scoring functions. The size of the space is exponential in the number of features/attributes considered. In consequence, any exploration of function space can quickly become very expensive, sufficiently so to be unusable in practice. A standard technique used in such scenarios is Monte Carlo simulation: rather than exploring every point in the space, we can choose a random sample. Monte Carlo methods require that this be an unbiased truly random sample. While it is straightforward to obtain a random sample in an ordinary multi-dimensional space, it turns out not to be so easy in function space, as we discuss in detail below. This paper addresses this need. While uniform function sampling may appear to be an esoteric technical exercise, it is key to responsible scoring, as we show below through multiple examples. Specifically, we consider three very different scenarios: (i) the design of a "fair" scoring function, (ii) assessment of a data set from a fairness viewpoint, and (iii) evaluation of analytical result stability as an inverse measure of "cherry-picking". Efficient function sampling is key to randomized algorithms for each.

In the preceding several paragraphs, we have considered the case of a machine learned scoring function, However, automated decision systems need not have their decision rules be learned by machine, Instead, they could be specified by human experts, as shown in Fig.~\ref{fig:arch}. Indeed, very many important deployed systems make automated decisions based on human expert specified rules and human expert specified scoring functions.

When humans specify scoring functions, they rarely seek precision,  instead, they specify something that seems "reasonable". In consequence, there is room for a computer system to proposed modifying their specification to achieve greater fairness. Any such proposal is likely to require exploration of function space and hence function sampling.

In summary, our main points are:
\begin{itemize}
    \item 
    Recognizing that many automated decision systems have human-specified scoring functions.
    \item
    Noticing that general solutions to achieve fairness, and other similar desirable properties, in scoring systems require efficient methods to sample from function space.
    \item
    Developing techniques for unbiased random sampling in function space, optionally constrained to a specified vicinity.~$\S$\ref{sec:sampler}
    \item
    Demonstrating the value of these techniques in diverse application scenarios.~$\S$\ref{sec:app}
    \item
    Proposing an efficient approximate construction of the arrangement of hyperplanes as a fundamental method, which builds upon function sampling, for responsible scoring mechanism design. \S~\ref{sec:arrangement}
\end{itemize}
We begin, in $\S$\ref{sec:pre}, with the formal problem set up, and some necessary background in computational geometry.
\section{Background} \label{sec:pre}
\subsection{Data and scoring model}
Our data set $\mathcal{D}$ comprises $n$ tuples. Each tuple $t \in \mathcal{D}$ is a  vector of $d$ scalar scoring attributes and zero or more additional non-scoring attributes, $\langle t[1], t[2], \ldots, t[d], t[non-scoring] \rangle$. 
In particular, some non-scoring attributes such as race and gender are considered to be {\em sensitive} and are used for measuring (un)fairness. Other non-scoring attributes that may be used for filtering.

We consider the general architecture of a score-based system to be as in Figure~\ref{fig:arch}.
The central component of the system is an evaluator that assign a score to each tuple in the input data
and uses it to generate the output by, for example, ranking or classifying the input. 
The score of a tuple is computed as a combination of its scoring attributes.


%
\begin{definition}[Scoring function]\label{def:scoringfunction}
A {\em scoring function} $f_{\vec{w}}:\mathbb{R}^d\rightarrow\mathbb{R}$, with weight vector $\vec{w} ~=~\langle w_1, w_2,\ldots,w_d \rangle$, assigns a score $f_{\vec{w}}(t) = \Sigma_{j=1}^d w_j t[j]$ to a tuple $t\in\mathcal{D}$.
When $\vec{w}$ is clear, we denote $f_{\vec{w}}(t)$ by $f(t)$.
\end{definition}
As shown in Figure~\ref{fig:arch}, the weights of a scoring mechanism could either be learned by machine or assigned by (human) experts.
The induced scores are used for evaluating (e.g. classifying or ranking) tuples.
The rank of a tuple is defined as its position in the sorted list of tuples based on their scores.
The scoring weights may be derived from a set of training data, typically using standard machine learning techniques such as linear regression or support vector machine.
However, many
well-known rankings, such as US News university ranking and FIFA rankings, are human-designed, i.e. scoring weights are assigned by experts. 

To further clarify the terms, let us introduce Example~\ref{example1}.

\begin{example}\label{example1}
Consider a real estate agency with two offices in {\sc Chicago, IL} and {\sc Detroit, MI}.
The owner assigns the agents based on need (randomly) to the offices.
At the end of the year, she wants to give a promotion to the ``best'' three agents.
The criteria for choosing the agents are $x_1:$ {\tt \small sales} and $x_2:$ {\tt \small customer satisfaction}.
Let the values in $\mathcal{D}$, after normalization, be as in Figure~\ref{fig:toyexample:data}.
The dataset contains $n=6$ tuples, over $d=2$ scoring attributes $x_1$ and $x_2$ and one non-scoring attribute {\tt \small location}, which in this example is considered to be the sensitive attribute. Following our notation, $t_3[2]$ refers to the value $x_2$ for $t_3$, which is $0.78$.

Suppose that, the two scoring attributes being (roughly) equally important, the owner chooses the weights $\vec{w}=\langle 1, 1 \rangle$ for scoring. That is, the score of every agent is computed as $f=x_1+x_2$.
The 5th column in Figure~\ref{fig:toyexample:data} shows the scores, based on this function.
The user's objective is ranking in this example, since she is interested in finding the top-3 tuples.
According to function $f$, the top-3 agents are $t_6$, $t_4$, and $t_2$, with scores 1.4, 1.38, and 1.37, respectively. 
Note that, according to $f$, all top-3 agents are located in Chicago and no agent from Detroit is selected.
\end{example}

\begin{figure}[!tb]
\centering
    \begin{tabular}{|l|c|c|@{}c@{}||@{}c@{}|@{}c@{}|}
	\hline
	\multicolumn{4}{|c||}{$\mathcal{D}$} & $f$&$f'$ \\ \hline
	id   & $x_1$ & $x_2$ & location&$ \; \langle 1, 1 \rangle \;$&$\langle 1.11, .9 \rangle$ \\ \hline \hline
	$t_1$& 0.63 & 0.71&Detroit&1.34&1.338 \\ \hline
	$t_2$& 0.72 & 0.65&Chicago&1.37&1.384 \\ \hline
	$t_3$& 0.58 & 0.78&Detroit&1.36&1.387 \\ \hline
    $t_4$& 0.7 & 0.68&Chicago&1.38&1.389 \\ \hline
	$t_5$& 0.53 & 0.82&Detroit&1.35&1.321 \\ \hline
	$t_6$& 0.61 & 0.79&Chicago&1.4&1.388 \\ \hline
	\end{tabular}
	\vspace{-1mm}\caption{Example~\ref{example1} -- Data}
    \label{fig:toyexample:data}
\end{figure}

Linear scoring functions are straightforward to compute and easy to explain~\cite{asudeh2016query}.
That is a reason those are popular for evaluation in general.
However, it turns out that the evaluations based on the scores highly depend on the choice of weights.
For instance, a ranking may significantly change by small changes in the weights.
Consider Example~\ref{example1}.
The owner chose the weight vector $\vec{w}=\langle 1, 1 \rangle$, simply because it would make sense to her, without paying attention to the consequences in terms of fairness.
However, small changes in the weights could dramatically change the ranking. 
For example, the function $f'$ with the weight vector $\vec{w'}=\langle 1.1, 0.9 \rangle$ may be equally good for the owner and she may not even have a preference between $\vec{w}$ and $\vec{w}'$. Probably her choice of weights is only because $\vec{w}$ is more intuitive to human beings.
The last column in Figure~\ref{fig:toyexample:data} shows the scores based on $f'$, which produce the ranking $f': \langle t_4, t_6, t_3, t_2, t_1, t_5\rangle$. Comparing it with the ranking generated by $f:  \langle t_6, t_4, t_2, t_3, t_5, t_1\rangle$, one may notice that the rank of each and every individual has changed.
More importantly, while according to $f$ all promotions are given to the agents of the Chicago office, $f'$ gives two promotions to Chicago and one to Detroit.

\begin{figure}[!tb]
\centering
    \includegraphics[width=0.45\textwidth]{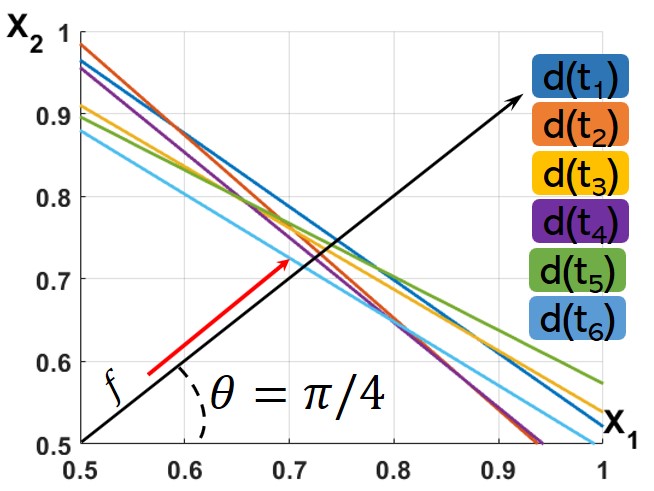}
    \caption{Example~\ref{example1}-Dual space}
    \label{fig:toyexample:dual}
\end{figure}

\subsection{Geometric interpretation}
{\em Primal space} is the popular geometric model for data, in which every attribute is modeled as a dimension and items are viewed as points in a multi-dimensional space.
Instead, we use a {\em dual space}~\cite{edelsbrunner} in $\mathbb{R}^d$, where an item $t$ is presented by a hyperplane $\mathsf{d}(t)$ given by the following equation of $d$ variables $x_1 \dots x_d$:
\begin{align}\label{eq:dual}
\mathsf{d}(t):~ t[1]\times x_1 + \dots + t[d]\times x_d = 1
\end{align}

Continuing with Example~\ref{example1}, Figure~\ref{fig:toyexample:dual} shows the items in the dual space. In $\mathbb{R}^2$, every item $t$ is a 2-dimensional hyperplane (i.e. simply a line) given by $\mathsf{d}(t): t[1] x_1 + t[2] x_2=1$.

A scoring function $f_{\vec{w}}$ is represented as a ray starting from the origin and passing through the point $[w_1, w_2,...,w_d]$.
For example, the function $f$ with the weight vector $\vec{w}=\langle 1,1 \rangle$ in Example~\ref{example1} is drawn in Figure~\ref{fig:toyexample:dual} as the origin-anchored ray that passes through the point $[1,1]$. Note that every scoring function (origin-anchored ray) can be identified by $(d-1)$ angles $\langle \theta_1, \theta_2, \cdots, \theta_{d-1} \rangle$, that can be computed using the polar coordinates of $w$.
For example, the function $f$ in Figure~\ref{fig:toyexample:dual} is identified by the angle $\theta=\pi/4$.

Consider the intersection of a dual hyperplane $\mathsf{d}(t)$ with the ray of a function $f$.
This intersection is in the form of $a\times\vec{w}$, because every point on the ray of $f$ is a linear scaling of $\vec{w}$.
Since this point is also on the hyperplane $\mathsf{d}(t)$,
$t[1]\times a\times w_1 + \dots + t[d]\times a\times w_d = 1$. Hence, $\sum t[j] w_j = 1/a$. 
This means that the dual hyperplane of any item with the score $f(t)=1/a$ intersects the ray of $f$ at point $a\times\vec{w}$.
As a result, the closer an intersection is to the origin, the higher is the score of its item.
Following this, the ordering of the items based on a function $f$ is determined by the ordering of the intersection of the hyperplanes with the vector of $f$. The closer an intersection is to the origin, the higher its rank.
For example, in Figure~\ref{fig:toyexample:dual}, the intersection of the line $t_6$ with the ray of $f=x_1 + x_2$ is closest to the origin, and $t_6$ has the highest rank for $f$.
Similarly, the border of a linear classifier can be viewed as a point on the ray of function that labels a tuple based on which side of the point in intersects the ray.


\section{Unbiased Function Sampling}\label{sec:sampler}

\begin{figure}[!tb]
        \centering
        \includegraphics[width = .45 \textwidth]{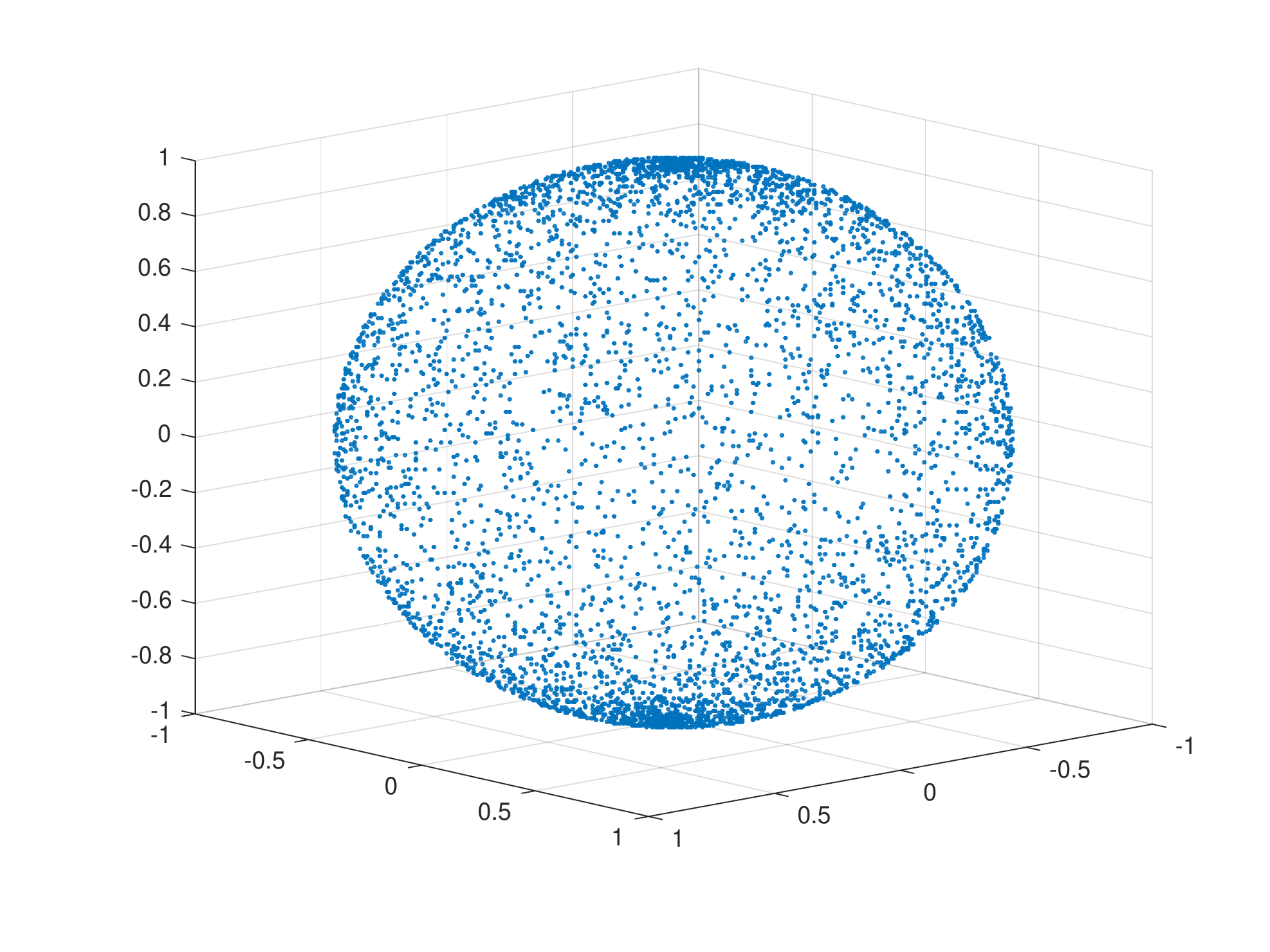}
        \caption{5000 random functions in $\mathbb{R}^3$, generated by uniformly sampling the angles}
        \label{fig:randf1}
\end{figure}

\begin{figure}[!tb]
        \centering
        \includegraphics[width = .45\textwidth]{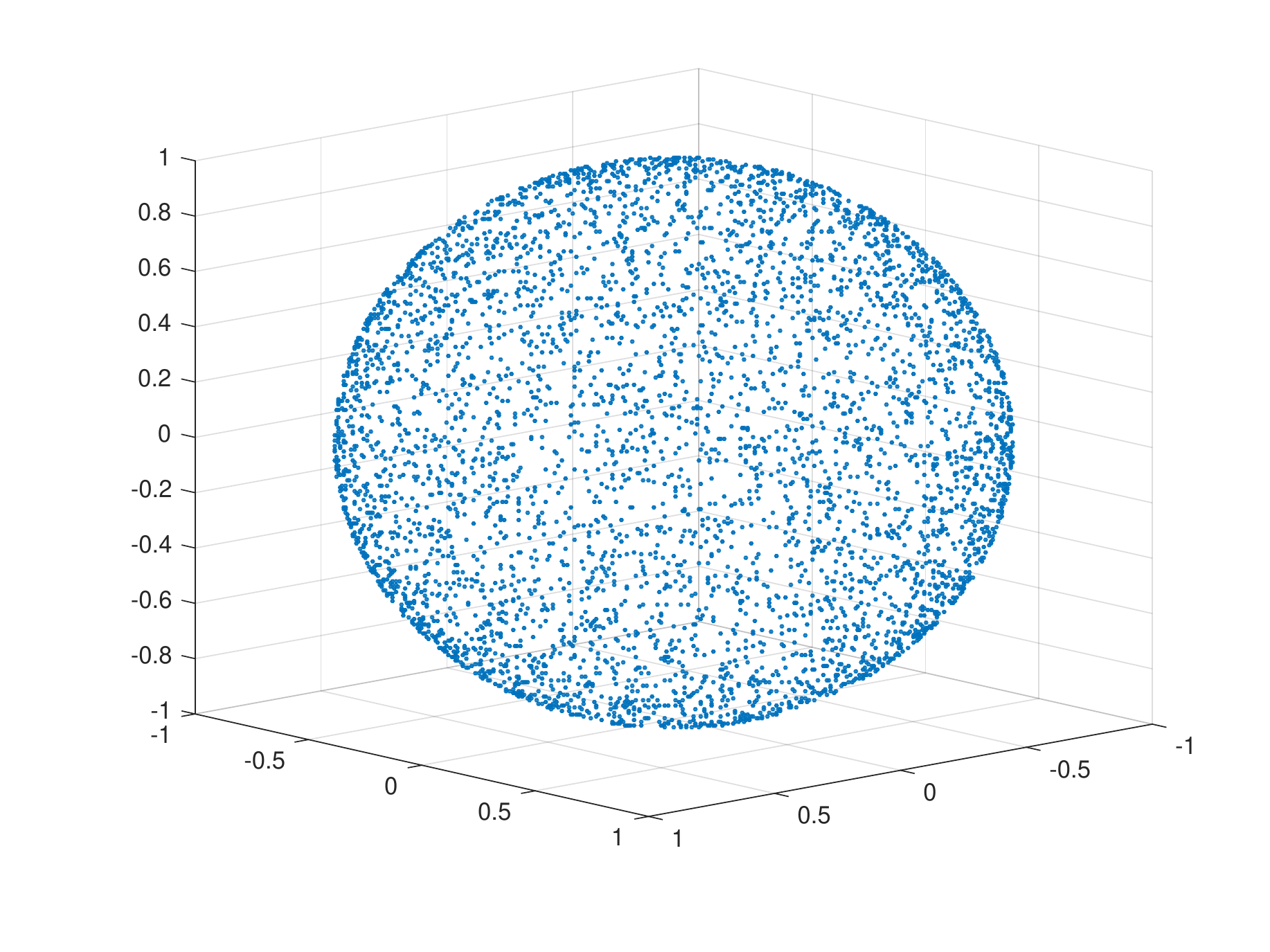}
        \caption{Illustration of 5000 random uniform functions taken in $\mathbb{R}^3$, using Algorithm~\ref{alg:sampu}}
        \label{fig:randf2}
\end{figure}

Unbiased sampling from the function space is the crucial step in 
developing randomized algorithms for the responsible scoring function design.
We will elaborate on this with multiple examples in \S~\ref{sec:app}.
In the following, we first discuss sampling from the complete function space and then propose an efficient sampler for \ar.

\subsection{Sampling from the entire function space}\label{subsec:sampleu}
Recall that every scoring function is identified as a vector of $d-1$ angles, one way of generating random functions is by generating angle vectors uniformly at random. This, however, as we shall show in the following, does not provide uniform random functions sampled from the function space, except for 2D.
First, let us propose Theorem~\ref{th:1} that establishes a key connection between sampling from the function space and sampling from the surface of unit d-sphere ($d$ dimensional hyper-sphere).

\begin{theorem}\label{th:1}
Uniform sampling of the point on the surface of the unit d-sphere provide uniform samples from the function space.
\end{theorem}
\begin{proof}
First, we note that there is a 1-1 mapping between the function space \fspace and the points on the surface of the unit d-sphere.
Every function  in \fspace is represented by an origin-anchored ray.
Every such ray $\theta = \{ \theta_1, \cdots,\theta_{d-1}\}$ passes through the point $\langle 1,\theta\rangle$ on the surface of the d-sphere.
Similarly, for every point $p=\langle 1,\theta\rangle$ there is only one scoring function, identified by the ray $\theta$ that passes through $p$.
As a result, using the points on the surface of the d-sphere to represent the scoring functions, sampling from the surface of the unit d-sphere samples the functions in \fspace.

Next, we need to show that such a sampling provides unbiased function sampling.
Consider the partitioning of the space of functions (origin-anchored rays) into Riemann d-cones.
Each cone is defined as a ray (passing through its center) $\rho_i$ and an angle $d\varphi$ around it. An unbiased sampler, should sample each of the cones with equal probability.
Now, consider the partitioning of the d-sphere into Riemann d-spherical sectors where all sectors have equal surface areas.
An unbiased sampler from the surface of the unit d-sphere samples the sectors with equal probabilities.
Each sector $s_i$ is identified by an origin-anchored ray (passing through the center of the sector) and an angle $d\varphi_i$. Because all sectors have equal areas, for two arbitrary sectors $s_i$ and $s_j$, $d\varphi_i = d\varphi_j = d\varphi$.
That is, these equi-area sectors are identified as the intersection of the equi-angle Riemann d-cones with the unit d-sphere.
As a result, since the sampler samples the sectors with equal probabilities, it samples the Riemann d-cones with equal probabilities.
Therefore, 
it samples scoring functions with equal probabilities. I.e., it is an unbiased sampler for scoring functions.
\end{proof}

We use the 1-1 mapping in Theorem~\ref{th:1} to demonstrate, in 3D, that
sampling functions by uniformly sampling the angles is not unbiased.
To do so, we 
generated a set of 5K samples, using this method.
The results are provided as plotted as the points on the surface of unit sphere in Figure~\ref{fig:randf1}.
Looking at the figure, it is easy to see that the distribution is not uniform, as the 
density of the end points reduces moving from the top and bottom to the middle.

Based on Theorem~\ref{th:1}, in order to generate unbiased samples from the function space, it is enough to sample (uniformly at random) from the surface of the d-sphere.
Hence, the problem of choosing functions uniformly at random from \fspace is equivalent to choosing random points from the surface of a $d$-sphere.
We \cite{asudeh2019designing,asudeh2018obtaining,asudeh2019rrr} adopt a method for uniform sampling of the points on the surface of the unit d-sphere~\cite{muller1959note, marsaglia1972choosing}.
Rather than sampling the angles, this method samples the weights using the {\em Normal distribution}, and normalizes them. 
This method works because the normal distribution function has a constant probability on the surfaces of d-spheres with common centers~\cite{marsaglia1972choosing,cramer2016mathematical}.
Algorithm~\ref{alg:sampu}, adopt this method to generate random functions from \fspace. 

\vspace{3mm}
\begin{algorithm}[!ht]
\caption{{\bf Sample\fspace} }
\begin{algorithmic}[1]
\label{alg:sampu}
    \FOR {$i=1$ to $d$} 
        \STATE $w_i = \mathcal{N}(0,1)$ \scriptsize{\tt // standard normal distribution}
    \ENDFOR
    \STATE {\bf return} $w/|w|$
\end{algorithmic}
\end{algorithm}

To demonstrate the uniformity of Sample\fspace, we used it to draw 5000 sample functions. Similar to Figure~\ref{fig:randf1}, we plotted the samples in Figure~\ref{fig:randf2}. The points are uniformly distributed in this figure.

\begin{figure}[!tb]
        \centering
        \includegraphics[width = 0.3\textwidth]{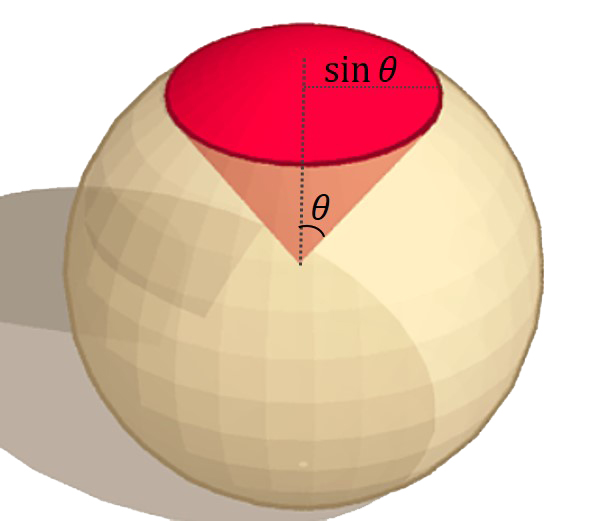}
        \caption{Modeling \ar as a unit $d$-spherical cap around the $d$-th axis}
        \label{fig:roi}
\end{figure}

\subsection{Sampling from a region of interest}\label{subsec:sampleui}

While sometimes, we require to sample from the (complete) function space, often it is the case that we want to limit the samples to the samples to the neighborhood of a given function. 
We define a region of interest \ar as the set of functions with minimum cosine similarity of at least $\cos(\theta)$ with the reference function $f$ (maximum angular distance of $\theta$ to the ray $\rho$ of $f$).
Our goal is design a sampler that: {\em given a region of interest \ar, generates uniform random samples from it}.

Given the unbiased sampler for the function space \fspace (Algorithm~\ref{alg:sampu}), an acceptance-rejection method~\cite{lucidl1989random} can be used for drawing samples from \ar.
The algorithm is straightforward: (i) draw a sample from the function space using Algorithm~\ref{alg:sampu}; (ii) if the drawn sample satisfies the cosine similarity constraint accept it, otherwise reject the sample and try again.

The major drawback of acceptance-rejection algorithms is that their efficiency depends on  on the acceptance probability $p$. Here $p$ is the volume ratio of \ar to \fspace. That is, the expected number of trials for drawing a sample for such probability is $1/p$.
Hence, this method is efficient if the volume of \ar is not small.

Therefore,
in the following, we alternatively propose an inverse CDF (cumulative distribution function) method~\cite{devroye1986sample} for generating random uniform functions from a region of interest. This method is preferred over the acceptance-rejection method when \ar has a small volume.

In order to design the sampler, following Theorem~\ref{th:1}, we model \ar as the surface unit $d$-spherical cap with angle $\theta$ around the d-th axis in $\mathbb{R}^d$ (Figure~\ref{fig:roi}).
This can be done using a rotation that maps the ray $\rho$ of $f$ to the d-th axis.
As we shall later show, after drawing a function sample, we will rotate the space back around $\rho$.

For an angle $\theta$, the plane $x_d=\cos \theta$ partitions the cap from the rest of the $d$-sphere.
Hence, the intersection of the set of the planes $\cos \theta\leq x_d\leq 1$ with the $d$-sphere define the cap.
The intersection of each such plane with the $d$-sphere is a $(d-1)$-sphere. For example, in Figure~\ref{fig:roi} the intersection of a plane, orthogonal to the z-axis, with the unit sphere is a circle (2-sphere).

At a high-level, in order to sample functions, we sample the points from the surface of such $(d-1)$-spheres proportionally to their areas, as explained in the following.
The surface area of a $\delta$-sphere with the radius $r$ is~\cite{li2011concise}:
\begin{align}\label{eq:area1}
A_\delta (r) = \frac{2\pi^{\delta/2}}{\Gamma(\delta/2)}r^{\delta-1}
\end{align}
where $\Gamma$ is the gamma function.

Using this equation, the area of the unit $d$-spherical cap can be stated as the integral over the surface areas of the $(d-1)$-spheres, 
defined by
the intersection of the planes $\cos \theta\leq x_d\leq 1$ with the $d$-sphere, as follows~\cite{li2011concise}:

\begin{align}\label{eq:area2}
A_d^{cap} (1) &= \int_0^\theta A_{d-1}\sin\phi d\phi
              = \frac{2\pi^{d/2}}{\Gamma(d/2)} \int_0^\theta \sin^{d-2}(\phi)d\phi
\end{align}

Therefore, considering the random angle $0 \leq x\leq \theta$, the cumulative density function (cdf) for $x$ is given by:

\begin{align}\label{eq:cdf1}
F(x) &= \frac{\int_0^x \sin^{d-2}(\phi)d\phi}{\int_0^\theta \sin^{d-2}(\phi)d\phi}
\end{align}

For a specific value of $d$, one can solve Equation~\ref{eq:cdf1}, find the inverse of $F$ and use it for sampling.
For instance, for $d=3$:
\begin{align}\label{eq:cdf3d}
F(x) = \frac{1-\cos x}{1-\cos\theta}
\Rightarrow F^{-1}(x) = \arccos \big(1-(1-\cos\theta)x\big)
\end{align}

For a general $d$, we can use the representation of $\int_0^\theta \sin^{d-2}(\phi)d\phi$ in the form of beta function and regularized incomplete beta function~\cite{li2011concise} and rewrite Equation~\ref{eq:cdf1} as\footnote{\small $I_z(\alpha, \beta)$ is the regularized incomplete beta function.}:
\begin{align}\label{eq:cdf}
F(x) &= \frac{I_{\sin^2(x)}\big(\frac{d-1}{2}, \frac{1}{2}\big)}{I_{\sin^2(\theta)}\big(\frac{d-1}{2}, \frac{1}{2}\big)}
\end{align}
However, since numeric methods are applied for finding the inverse of the regularized incomplete beta function~\cite{cran1977remark}, we consider a numeric solution for Equation~\ref{eq:cdf1}.
Consider a regular partition of the interval $[0,\theta]$ to its Riemann pieces. The integral $\int_0^\theta \sin^{d-2}(\phi)d\phi$ can be computed as the sequence of Riemann sums over the partitions of the interval.
We apply this for computing both the denominator and the nominator of Equation~\ref{eq:cdf1}.
Given the partition of the interval, we start from the angle 0, and for each partition $x'$, compute the value of $\int_0^{x'} \sin^{d-2}(\phi)d\phi$ as the aggregate over the previous summations and store it in a sorted list.
As a result, in addition to the value of the denominator, we have the value of $F$ for each of the partitions. We will later apply binary search on this list, in order to find the angle $x$ that has the area $F(x)$.
Algorithm~\ref{alg:riemann} shows the pseudocode of the function {\it RiemannSums} that computes the denominator and returns the list of partial integrals divided by the denominator.
In addition to the angle $\theta$, the function takes the number of partitions as the input.

\begin{algorithm}[!h]
\caption{{\bf RiemannSums}\\
{\bf Input:} The angle $\theta$ and number of partitions $\gamma$
}
\begin{algorithmic}[1]
\label{alg:riemann}
    \STATE $\epsilon = \theta/\gamma$
    \STATE $L=[0]$; $A=0$; $\alpha= \epsilon$
    \FOR{$i = 1$ to $\gamma$}
        \STATE $A= A+\sin^{d-2}(\alpha)$
        \STATE $L.append(A)$
        \STATE $\alpha = \alpha + \epsilon$
    \ENDFOR
    \STATE {\bf for} $i=1$ to $\gamma$ {\bf do} $L[i] = L[i]/A$
    \STATE {\bf return} $L$
\end{algorithmic}
\end{algorithm}

Algorithm~\ref{alg:sampui} shows the pseudocode of the inverse CDF sampler.
As an example, consider the case in $\mathbb{R}^3$ where the objective is to generate random numbers around the ray $(\pi/6,\pi/4)$ with angle $\theta=\pi/20$.
The algorithm starts by drawing a random uniform number in range $[0,1]$. 
Let such a random number be 0.13.
It takes the list $L$ (computed using the function RiemannSums) as the input and draws a random function from \ar.
To do so, it first draws a random uniform number $y$ in the range [0,1].
Next, it applies a binary search on the list of partial integrals the index $i$ where $F(x_i)=y$.
Considering a fine granularity of the partitions, we assume that the areas of all $(d-1)$-spheres inside each partition are equal. Hence, the algorithm selects a point in the partition (Line 3 of the algorithm) uniformly at random.
Obviously, instead, the algorithm can use the equation of the inverse function.
Continuing with our example, while using Equation~\ref{eq:cdf3d}, the corresponding $y$ value for 0.13 is $\pi/55.5$.

\begin{algorithm}[!h]
\caption{{\bf Sample \ar}\\
{\bf Input:} The ray $\rho$, angle $\theta$, and the list $L$}
\begin{algorithmic}[1]
\label{alg:sampui}
    \STATE $y = U[0,1]$ {\scriptsize \tt // draw a uniform sample in range [0,1]}
    \STATE $i =$ {\bf binarySearch}$(y,L)$
    \STATE $x = i\times \epsilon +U[0,\epsilon]$ {\scriptsize \tt // add a small noise}
    \STATE {\bf for} $i=1$ to $d-1$ {\bf do} $\hat{w}_i = \mathcal{N}(0,1)$
    \STATE $\langle \theta_1,\cdots,\theta_{d-2} \rangle = $ the angles in polar representation of $\hat{w}$
    \STATE $w = $ {\bf toCartesian}$(1, \langle \theta_1,\cdots , \theta_{d-2},x \rangle)$
    \STATE {\bf return Rotate}($w$, $\rho$)
\end{algorithmic}
\end{algorithm}

Recall that the angle $x$ specifies the intersection of a plane with the $d$-spherical cap, which is a $(d-1)$-sphere.
Hence, after finding the angle $x$, we need to sample from the surface of a $(d-1)$-sphere, uniformly at random.
For our example in $\mathbb{R}^3$, the intersection is a circle (2-sphere) and, therefore, we need to sample from the surface of the circle.
Also, recall from \S~\ref{subsec:sampleu} that the normalized set of $d-1$ random numbers drawn from the normal distribution provide a random sample point on the surface of $(d-1)$-sphere.
The algorithm Sample\ar uses this for generating such a random point.
It uses the angle combination of the drawn random point from the surface of a $(d-1)$-sphere and combines them with the angle $x$ (with the $d$-th axis).
In our example in $\mathbb{R}^3$, let the sampled point on the circle have the angle $0.8\pi$. Hence, the angle combination is $\langle 0.8\pi, \pi/55.5\rangle$.
After this step, the point specified by the polar coordinates $(1, \langle \theta_1,\cdots , \theta_{d-2},x \rangle)$ is the random uniform point from the surface of $d$-spherical cap around the $d$-th axis.
As the final step, the algorithm needs to rotate the coordinate system such that the center of the cap (currently on $d$-th axis) falls on the ray $\rho$.
We rely on the existence of the function {\it Rotate} for this, explained in \S~\ref{sec:rotation}.

\begin{figure}[!tb]
        \centering
        \includegraphics[width = 0.5\textwidth]{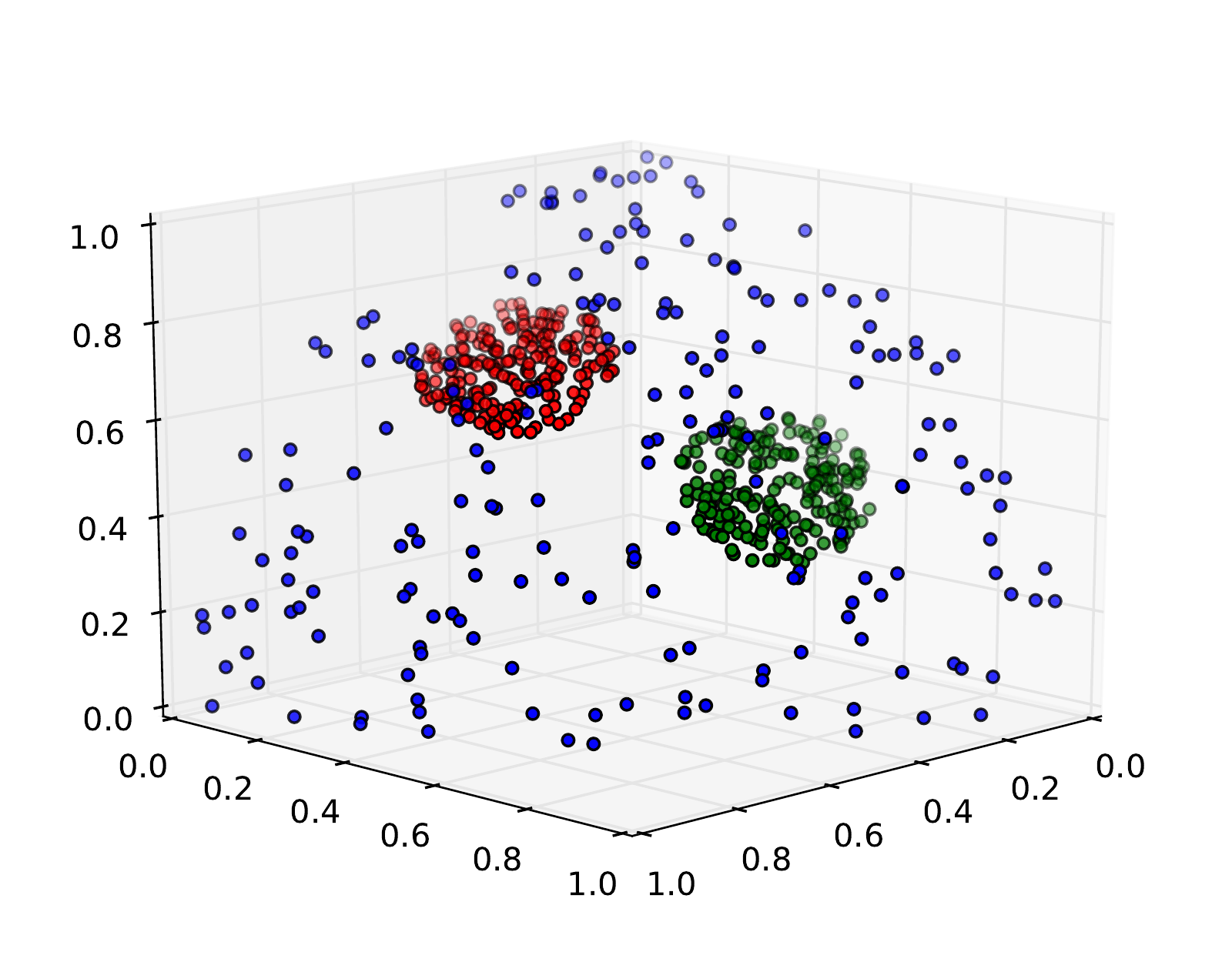}
        \caption{Samples generated using (i) blue: Algorithm~\ref{alg:sampu}, (ii) green: Algorithm~\ref{alg:sampui} and Algorithm~\ref{alg:riemann}, (iii) red: Algorithm~\ref{alg:sampui} and Equation~\ref{eq:cdf3d}}
        \label{fig:roi2}
\end{figure}

Figure~\ref{fig:roi2} shows three cases of 200 samples in $\mathbb{R}^3$ where (i) the blue points (scattered over the space) are sampled from the first quadrant of function space using Algorithm~\ref{alg:sampu}, (ii) green points (the right cluster) are generated around the ray $(\pi/3,\pi/3)$ with angle $\pi/20$ using Algorithm~\ref{alg:sampui} and Algorithm~\ref{alg:riemann}, and (iii) red points (the left cluster) are generated around the ray $(\pi/6,\pi/4)$ with angle $\pi/20$ using Algorithm~\ref{alg:sampui} while using Equation~\ref{eq:cdf3d} for the inverse CDF.


\vspace{2mm}
\subsubsection{Coordinate system rotation}~\label{sec:rotation}
In order to generate random functions in \ar, Algorithm~\ref{alg:sampui} models the region of interest as a $d$-spherical cap around the $d$-th axis.
Therefore, after picking a random vector $w$, it needs to {\em rotate} the space such that the $d$-th axis gets aligned on the input vector $\rho$. This moves the drawn sample to the region of interest.
We do the rotation, using a so-called ``transformation matrix''~\cite{baker2012matrix}. 
Having such a $d\times d$ rotation matrix $M$, the result of rotation on a vector $w$ is a vector $w'$, generated as $w' = Mw$.
For example in $\mathbb{R}^2$, the following matrix rotates the coordinate system counterclockwise to an angle of $\theta$:
 \[
   M=
  \left[ {\begin{array}{cc}
   \cos \theta & -\sin\theta \\
   \sin\theta & \cos\theta \\
  \end{array} } \right]
\]

We use this matrix for deriving the rotation matrix we are looking for.
The idea is that instead of applying the the rotation at once, we can do the rotation on axes separately. For example, for $\mathbb{R}^3$, we first can fix the z-axis and do the rotation on the x-y plane and then fix the y-axis and do the rotation on the x-z plane.

The $d$ by $d$ matrix $M_i$, specified in Equation~\ref{eq:mi}, rotates the coordinate system on the $x_1$-$x_{i+1}$ plane counterclockwise to an angle of $\rho_i$. All the values in $M$ except the diameter, $M[1,i+1]$, and $M[i+1,1]$ are zero. Also, all the values on the diameter, except $M[1,1]$ and $M[i+1,i+1]$ are one.

\begin{align}\label{eq:mi}
M_i = 
\begin{blockarray}{ccccccc}
1 & 2& \cdots & i+1&\cdots& d \\
\begin{block}{(cccccc)c}
  \cos\rho_i & 0 & \cdots & -\sin\rho_i &\cdots & 0& 1 \\
  0 & 1 & \cdots & 0 &\cdots & 0 & 2 \\
  \vdots & \vdots & \ddots & \vdots &\vdots & \vdots & \vdots \\
  \sin\rho_i & 0 & \cdots & \cos\rho_i &\cdots & 0 & i+1 \\
  \vdots & \vdots & \vdots & \vdots &\ddots & \cdots & \vdots \\
  0 & 0 & \cdots & 0 &\cdots & 1 & d \\
\end{block}
\end{blockarray}
 \end{align}

The last point is that in order to rotate the $d$-th axis to $\rho$, on all $x_1$-$x_{i+1}$ planes except the last one the rotations are counterclockwise, while for the $x_1$-$x_d$ plane the rotation is clockwise. We change $\rho_{d-1}$ to $(\pi/2 - \rho_{d-1})$ to make the last rotation also counterclockwise.

\vspace{3mm}
\begin{algorithm}[!h]
\caption{{\bf Rotate} \\
		 {\bf Input:} vector $w$ and ray $\rho$ (in form of $d-1$ angles)\\
		 {\bf Output:} vector $w'$
		}
\begin{algorithmic}[1]
\label{alg:rotate}
	\STATE $w' = w$
    \STATE $\rho_{d-1} = \pi/2 - \rho_{d-1}$
    \FOR{$i=d-1$ down to $1$}
    	\STATE compute $M_i$, using Equation~\ref{eq:mi}
    	\STATE $w' = M_i\times w'$
    \ENDFOR
    \STATE {\bf return} $w'$
\end{algorithmic}
\end{algorithm}

\section{Application Demonstration}\label{sec:app}
Following the geometric interpretation in \S~\ref{sec:pre}, due to the complexity of function space, exact algorithms are usually not scalable.
Recall that the intersection of the dual hyperplanes of tuples with a ray of function $f$ determines scores based on $f$.
For ranking, the ordering of the intersections identify the ranking.
For classification, which side of the border-point a dual hyperplanes cuts the ray of the function shows the class label.
In both of these cases the function space gets partitioned, by an arrangement~\cite{edelsbrunner} of a set of hyperplanes.
As a result, the exact algorithms for exploring the function space is cursed by the number of dimensions, i.e., exponentially depend on the number of scoring attributes.

Fortunately, 
unlike inefficiency of the exact algorithms, efficient and effective randomized algorithms can be designed.
The function sampling introduced in \S~\ref{sec:sampler}, provides a powerful tool for designing randomized approximation algorithms such as Monte-carlo estimation.
In the following, we briefly discuss three diverse applications in the context of ranking.

First in \S~\ref{sec:fairness}, we will use function sampling to {\em design} a fair scoring function.
Our proposal is a tool for assisting scoring function designer, the {\em human-in-the-loop}.
Next in \S~\ref{sec:hardness}, we will use our sampler to {\em evaluate the fitness} of a dataset for designing fair rankings.
Finally, in \S~\ref{sec:stability}, we will show how to use function sampling for {\em audit} if a scoring function has been (intentionally or unintentionally) cherry-picked for ranking and to obtain the most reliable (stable) ranking in the vicinity of the given scoring function.



\subsection{Assisting human experts to achieve fairness by design in ranking}\label{sec:fairness}
Fairness of decision systems has been receiving tremendous attention.
We interpret fairness to mean that (a) disparate impact, which may arise as a result of historical discrimination, needs to be mitigated; and yet (b) disparate treatment cannot be exercised to mitigate disparate impact when the decision system is deployed.
Following these, \cite{asudeh2019designing} proposes a system that helps human expert to design a fair scoring function, such that the ranking based on it 
{\em both} mitigate disparate impact {\em and} do not exercise disparate treatment (by not explicitly using information about an individual's membership in a protected group) during deployment.
That is, a single scoring function will be used for all items in the dataset, irrespective of their membership in a protected group.
At a high level, the expert first designs an initial scoring function $f$, using her domain knowledge.
The system, then, finds the most similar function $f'$ to $f$ such that its output is ``fair''.
Adopting a general definition, a ranking is considered fair if it satisfies a set of criteria.
For instance, in Example~\ref{example1}, the initial function $f$ has the weight vector $\vec{w}=\langle 1,1\rangle$.
Assume the owner knows that, because of some hidden factors, sales and customer satisfaction patterns are different in Chicago and Detroit. Hence, she considers the selection of the top-3 agents to be fair, if it assigns at least one of the promotions to each one of the offices. Note that according to this criterion, the ranking provided by $f=x_1+x_2$ is not fair as it assigns all three promotions to the agents in Chicago.
On the other hand the ranking generated by function $f'=1.1x_1+.9x_2$ assigns two of the promotions to Chicago and one to Detroit, and hence is considered to be fair. 
In this example, given the initial function $f$, the system returns the most similar function $f''$ to $f$ that assigns at least one of the promotions to each of the offices.

In Figure~\ref{fig:toyexample:dual}, 
consider the intersection of two dual lines, for example, $\mathsf{d}(t_1)$ and $\mathsf{d}(t_2)$ and the origin-anchored function passing through this intersection. Every function in the top-left of this ray ranks $t_1$ higher than $t_2$, while for every function in its bottom-right $t_2$ is ranked higher. Referring to this intersection as the ordering-exchange between the two tuples, the function space partitions into regions such that all functions in a region produce the same ranking.
In higher dimensions, the hyper-plane $\sum_{k=1}^d(t_i[k]-t_j[k])~x_k=0$ identifies the ordering-exchange between $t_i$ and $t_j$.
In ordering to find satisfactory functions (that generate fair rankings)
\cite{asudeh2019designing} reprocess the function space to identify satisfactory regions.
Unfortunately, finding the satisfactory regions require the construction of arrangement~\cite{edelsbrunner} of hyperplanes, which has a exponential complexity to the number of dimensions.
Several approximations and heuristics are proposed to solve the problem, which still don't scale beyond thousands of items.

The function sampling proposed in \S~\ref{sec:sampler} enables a practical solution for the problem, in particular for on-the-fly query answering.
Given a region of interest, identified by the initial function of the user and an angle $\theta$, we can use function samples for discovering a (function in a) satisfactory region.

The idea is to take an unbiased function sample from the region of interest, check if the ranking generated by it satisfies the fairness criteria, and if so return it to the user, otherwise take another sample and continue\footnote{We demonstrate using this technique for finding fair and stable rankings in \cite{guan2019mithraranking}.}.
Let $s$ be the sampling budget. For every sample, generating the ranking is in $O(n\log n)$. Therefore, the algorithm is in $O(s~n\log n)$. Also, since the samples are taken from the $\theta$-vicinity of the user input, given that the optimal function has an angle distance $\epsilon\geq 0$ from the user input, the output of this algorithm (if it can find a satisfactory function) is within an additive approximation $\theta$ of the optimal solution.

Using $s$ as the sampling budget, the algorithm is expected to discover a satisfactory function if the volume ratio of the satisfactory regions to the volume of the region of interest is more than $1/s$.
We note that this strategy is not effective for the discovery of very small satisfactory regions. Still, as we shall discuss in \S~\ref{sec:stability} the rankings produced in these regions are not reliable as small changes in the weight vector changes the outcome.

\subsection{Evaluating the fitness of a dataset of fair ranking}\label{sec:hardness}
In previous section, we explained, given a dataset, how to design scoring function such that their rankings are fair (satisfy some fairness constraints).
It however is not possible to find satisfactory functions.
If the disparity in the input data is high or the fairness constraints are too restrictive, it may even be impossible to find such a function.

Consider an affirmative action (and implicit bias) scenario where a company board is worried that too few employees of a particular race are being hired (disparate impact) through applicant screening software.
A first response may be to have a public relations statement about there being too few qualified applicants of that race.
A more in-depth analysis may show that each applicant has some value for each of 50 attributes considered by the software, and it is not the case that applicants of any race are dominated in all desired attributes.  However, the attributes in which members of the protected race tend to score more highly are given less weight than those on which they tend to score less well. This is bias in the screening software, which could be due to historical artifacts in the training data, implicit bias of the algorithm designer, or some other cause.

Our goal in this section is to assess {\em how ``hard''} it is to design a scoring function such that its output is fair.
We wish to distinguish between the case where members of the protected race are truly dominated from the case where their desired attributes are undervalued.
In the following, first we introduce a measure\footnote{We would like to emphasize that even though we use ranking for application demonstration, 
our proposed metrics and sampling-based solutions can be adopted for classification.} to quantify the hardness. We then use function sampling to develop a Monte-Carlo method~\cite{montecarlo} that computes the measure efficiently.

Given the need for measures that show the hardness of fairness, for different algorithmic tasks, we do not limit our definitions to the scoring functions, but rather define them over the large set of algorithmic tasks with a collection of valid settings.
While setting up an algorithm to get applied on a dataset, each of the valid settings may result in either a fair or unfair output.
The more the number of unfair outputs are, the harder and less likely is to design a meaningful setting that is fair.
Using this observation, we first propose the {\em unfair portion} (\up) measure, Definition~\ref{def:up}, that measure the portion of the parameter settings that result in unfair outputs.

\begin{definition}[Unfair Portion (\up)]\label{def:up}
Consider a dataset $\mathcal{D}$, a fairness criteria $\phi$, and an algorithm with a collection of parameter settings $\mathcal{F}$.
For each setting $f\in\mathcal{F}$, let $\mathcal{A}_f(\mathcal{D})$ be the output of $A$ on $\mathcal{D}$ using $f$.
The unfair portion of $\mathcal{F}$ on $\mathcal{D}$ based on $\phi$ is defined as the ratio of settings $f\in\mathcal{F}$ that produce a fair output based on $\phi$. Formally:
\begin{align}\label{eq:up}
\up_\mathcal{A}(\mathcal{D},\mathcal{F}, \phi) = \frac{|\{f\in\mathcal{F}~|~ sat(\mathcal{A}_f(\mathcal{D}),\phi)=\bot\}|}{|\mathcal{F}|}
\end{align}
\end{definition}

The collection of parameter settings, the denominator of Equation~\ref{eq:up} for linear scoring functions, can either be a discrete {\em work-load} of functions, or a region of interest, or even the complete function space.

As explained in the previous section, the function space is partitioned into a set of ranking regions that each generate a unique ranking. An exact solution for computing \up for linear scoring functions, requires to first find these region (by construction the arrangement of ordering-exchanges), for each region generate the ranking, and finally compute the volume of the satisfactory regions. Note that even computing an exact volume computation in higher dimensions is not practical.

Fortunately, the unbiased function sampler proposed in \S~\ref{sec:sampler} enable designing an efficient and accurate Monte-Carlo estimation for the problem.
Monte-Carlo methods work based on repeated sampling and the central limit theorem in order to solve deterministic problems. We consider using these algorithms for numeric integration.
Based on the law of large numbers~\cite{durrett2010probability},
the mean of independent random variables can be used for approximating the integrals.

The algorithm is designed based on the fact that the probability that a sampled function, drawn uniformly at random from a region of interest, hits each ranking region proportional to the volume of the region. 
We use this for designing the Monte-Carlo estimator.

Let $\chi_{\bar{\Phi}}$ be the Bernoulli variable that is 1 if the output of an function does not satisfy the fairness constraints and is $0$ otherwise.
Also, let $R_{\bar{\Phi}}$ be the union of satisfactory regions in a region of interest \ar.
For a sampled function, drawn uniformly at random from \ar, the probability mass function (pdf) of $\chi_{\bar{\Phi}}$ is defined as:

\begin{align}\label{eq:functiondist}
p(\chi_{\bar{\Phi}})=\begin{cases}
\frac{vol(R_{\bar{\Phi}})}{vol(\mathcal{F}^*)}  &\chi_{\bar{\Phi}}=1 \\
1 - \frac{vol(R_{\bar{\Phi}})}{vol(\mathcal{F}^*)} &\chi_{\bar{\Phi}}=0
\end{cases}
\end{align}

Using Definition~\ref{def:up}, Equation~\ref{eq:functiondist} can be rewritten as:

\begin{align}\label{eq:functiondist-up}
p(\chi_{\bar{\Phi}})=\begin{cases}
\up &\chi_{\bar{\Phi}}=1 \\
1 - \up &\chi_{\bar{\Phi}}=0
\end{cases}
\end{align}
Since $\chi_{\bar{\Phi}}$ is a Bernoulli variable, its mean and standard deviation are $\mu_{\chi_{\bar{\Phi}}} = \up$ and $\sigma_{\chi_{\bar{\Phi}}} =\up/(1-\up)$, respectively.

Let $\bar{\chi}$ be the average of $\chi_{\bar{\Phi}}$ over $s$ samples.
Based on the central limit theorem, $\bar{\chi}$ follows the Gaussian distribution N$\big(\mu_{\chi_{\bar{\Phi}}}, \frac{\sigma_{\chi_{\bar{\Phi}}} }{\sqrt{s}}\big)$ -- the Gaussian distribution with the mean $\mu_{\chi_{\bar{\Phi}}} = \up$ and standard deviation $\frac{\sigma_{\chi_{\bar{\Phi}}}}{\sqrt{s}}$.

For a confidence level $\alpha$ and error $e$, the confidence range $[\bar{\chi}-e, \bar{\chi}+e]$ is the range, where:
$$p(\bar{\chi}-e\leq \mu_{\chi_{\bar{\Phi}}}\leq \bar{\chi}+e)= 1-\alpha$$

Using the Z-table:
\begin{align} \label{eq:costerror}
\nonumber
e &= Z(1-\frac{\alpha}{2})\frac{\sigma_{\chi_{\bar{\Phi}}}}{\sqrt{s}}\\
  &= Z(1-\frac{\alpha}{2})\sqrt{\frac{ \bar{\chi}(1-\bar{\chi})}{s}}
\end{align}

\begin{algorithm}[H]
\caption{\up \\
         {\bf Input:} The dataset $\mathcal{D}$, fairness constraints $\phi$, region of interest \ar: $\langle \rho, \theta \rangle$, confidence level $\alpha$\\
         {\bf Output:} \up
        }
\begin{algorithmic}[1]
\label{alg:hardness}
    \STATE $cnt = 0$
	\FOR {$k=1$ to $s$} 
    	\STATE $\Vec{w} =$Sample\ar$(\rho, \theta)$
    	\STATE $\Delta = $ Rank $\mathcal{D}$ based on $f_{\vec{w}}$
        \IF{$\Delta$ does not satisfy fairness constraints $\phi$} 
        	\STATE $cnt = cnt+1$
        \ENDIF
    \ENDFOR
    \STATE $\up= \frac{cnt}{s}$
    \STATE e $= Z(1-\frac{\alpha}{2})\sqrt{\frac{\up(1-\up)}{s}}$
    \STATE {\bf return} (\up,e)
\end{algorithmic}
\end{algorithm}

Following the above discussion, Algorithm~\ref{alg:hardness} uses a budget of $s$ function samples and estimates \up.
Using a set of $s$ uniform samples from the region of interest, it counts the number of unsatisfactory functions and computes \up and confidence error as in Equation~\ref{eq:costerror}.

For each function sample, assigning scores to the tuples, ordering them, and judging if the output is fair is in $O(n\log n)$. Therefore, the algorithm is in $O(s~n\log n)$.

\subsection{Auditing Cherry-picked (and Obtaining Stable) Rankings}\label{sec:stability}
After studying the application of function sampling for fairness in ranking, we now show how the function sampler we proposed in \S~\ref{sec:sampler} enables an efficient way for detecting ``cherry-picked'' rankings and also for obtaining ``stable rankings''~\cite{asudeh2018obtaining}.
In Example~\ref{example1}, we observed that small changes in the weight vector of the scoring function can cause a significant change in ranking.
The sensitivity of rankings to the choice of weight has been widely recognized.
For example,  Malcolm Gladwell has nicely described this issue in the context of (highly criticized) university rankings in \cite{Gladwell11Order}.
He argues that score-based rankings are popular for university rankings, because {\em a single score} allows us to easily judge between entities. Yet, he highlights that rankings depend on the weights chosen for scoring function, hence not reliable.

who should be able to assess the reliability of rankings generated by a scoring function $f$. 
If small changes in the weight vector change the ranking, the generated ranking is not reliable.
In other words, if only a small portion of function in a region of interest around $f$ generate the ranking generated by $f$, it may suggest that the ranking was engineered or cherry-picked by the producer to obtain a specific outcome. 
On the other hand, if a ranking is generated by a large region of functions, the ranking is stable, not sensitive to perturbation in the weights, hence reliable for decision making.

Similar to \S~\ref{sec:fairness} and~\ref{sec:hardness}, measuring if a ranking has been cherry-picked, as well as finding stable rankings in a region of interest \ar requires identifying ranking regions, and in addition to compute volumes of the regions.
\cite{asudeh2018obtaining} design a threshold-based arrangement construction that, at any step, concentrates on extending the construction in the largest region, delaying the others.
Even though this algorithm is efficient when there are fairly large (stable) regions in \ar, it fails when all regions have similar volumes. Also, being cursed by dimensionality, the algorithm does not scale beyond a few dimensions.

Despite the complexity of exact algorithms, once again, function sampling provides an efficient and accurate approach both for auditing cherry-picked rankings and for obtaining stable rankings.
Similar to \S~\ref{sec:hardness}, the proposed algorithm is a Monte-Carlo estimation. 
The auditing algorithm is similar to Algorithm~\ref{alg:hardness} that issues $s$ unbiased samples from \ar, and counts the appearance of the ranking produced by the reference function provided by the user.
Also, defining a Bernoulli variable that is 1 if the reference ranking is observed, it computes the confidence error, using Equation~\ref{eq:costerror}.

In order to find the most stable rankings, $s$ function samples from \ar, the algorithm generates the output for each of the sampled functions and, using a hash data structure, maintains the occurrence history of each of the rankings. The ranking that appears the most is the function estimate for the most stable ranking in \ar. 

In order to demonstrate the efficiency and effectiveness of proposed function-sampling based algorithms, in the following we provide two experiment results for auditing cherry-picked rankings and obtaining stable ones~\cite{asudeh2018obtaining}.

\begin{figure}[!tb]
        \centering
        \includegraphics[width=.45\textwidth]{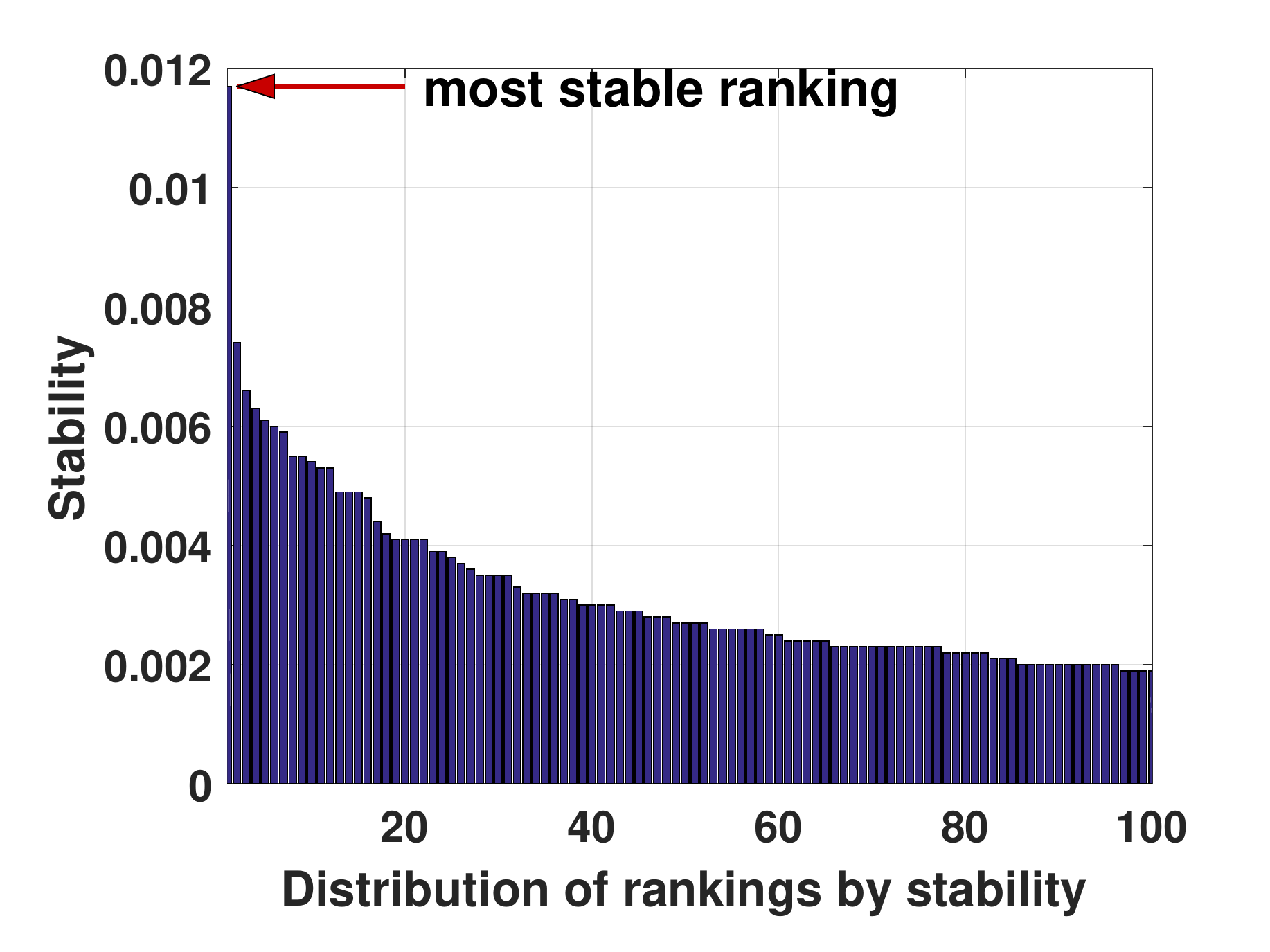}
        \caption{FIFA Rankings: stability around reference vector $\langle 1, 0.5, 0.3, 0.2 \rangle$ with 0.999 cosine sim.}
        \label{fig:fifa1}
\end{figure} 

\subsubsection{Stability in FIFA Rankings}
Many sports use ranking schemes. An example is the FIFA (the international football association) World Ranking of national soccer teams based on recent performance~\cite{fifa1}.
The Rankings are based on a human-designed scoring function where the score of a team $t$ depends on team performance values for $A_1$ (current year), $A_2$ (past year), $A_3$ (two years ago), and $A_4$ (three years ago).  The given score function, from which the reference ranking is derived, is: $t[1] + 0.5 t[2] + 0.3 t[3] + 0.2 t[4]$.
FIFA relies on these rankings for modeling the progress and abilities of the national-A soccer teams~\cite{fifa} and to seed important competitions in different tournaments, including the 2018 FIFA World Cup.
Despite the trust placed by FIFA in these rankings, many critics have questioned their validity.
Therefore, we chose it and used function sampling to audit its stability.
We consider the top 100 teams in our experiment and focus on a narrow region of interest \ar defined by 0.999 cosine similarity ($\theta=\pi/100$) around the reference weights used by FIFA, i.e. $w=\langle 1, 0.5, 0.3, 0.2 \rangle$.
We used function 10,000 unbiased sampling drawn from \ar and produced the top-100 stable rankings around the reference weight vector. Figure~\ref{fig:fifa1} shows the distribution of stable rankings.  
First, one can see that even in such a narrow region, there exists many rankings. Second, there exists a ranking that is relatively much more stable than the others.
Perhaps the most interesting observation is that {\em the reference ranking did not appear in the top-100 stable rankings}.
While FIFA advertises this ranking as ``a reliable measure for comparing the national A-teams''~\cite{fifa}, 
our finding questions FIFA's trust in such an unstable ranking for making important decisions such as seeding the world cup.

\subsubsection{Efficiency advantage of function-sampling for stability}:
In order to study the benefit of function sampling for obtaining stable rankings, we conducted two experiments on a flight dataset from the US Department of Transportation (DoT)~\cite{DoT}. We collected a set of 1,322,023 records, for the flights conducted by the 14 US carriers in the last three months of 2017. We consider the attributes {\tt air-time}, {\tt taxi-in} and {\tt taxi-out} for ranking.
We used the threshold-based algorithm~\cite{asudeh2018obtaining} and function sampling for obtaining the stable rankings, while considering a budget of $5,000$ samples from \ar. We considered the number of attributes $d=3$, the width of the region of interest $\theta=\pi/50$, and $k=10$, that is we studied the stability of the top-k sets, not the entire ranking.
While finding the most stable ranking for the input size of $n=10K$, using the threshold algorithm required almost 3 hours, the randomized algorithm required a few seconds for this setting and a few minutes for $n=100K$.

\section{An Efficient Approximation for Arrangement Construction}\label{sec:arrangement}
Constructing the arrangement of a set of hyper-planes is often a key step in designing algorithms for responsible scoring mechanisms, as we saw in the preceding section.
Unfortunately, due to the exponential complexity of the arrangement to the number of scoring attributes, the algorithms based on arrangement construction are not efficient.
In this section we propose an efficient approximation algorithm for arrangement construction, using the function sampling proposed in \S~\ref{sec:sampler}.
Using function sampling, the following method efficiently constructs the regions that are not very small, {\em independent of the number of dimensions}. This is particularly useful, knowing the very small regions generate cherry-picked outputs and, hence, are already not of our interest.

First, let us review how the algorithm for construction of an arrangement works~\cite{edelsbrunner}: at a high-level, the algorithm adds hyperplanes one after the other to the space and updates the arrangement accordingly.
The first hyperplane $h_1$ partitions the space in to halfspaces $h_1^-$ and $h_1^+$. 
At any step, in order to add the next hyperplane $h_i$, we need to identify the regions that intersect with the new hyperplane and cut each such regions into two smaller regions, using $h_i$.

Identifying if a hyperplane intersects with a region requires to solve a linear programming.
Instead, we propose to use a set of function samples $\mathcal{S}$ for this purpose.
The idea is that a hyperplane $h$ intersects a region $R$, if there exists two points (function in our context) $w$ and $w'$ in $R$ such that $w$ belongs to $h^-$ and $w'$ is in $h^+$, i.e. $\sum_{i=0}^d h[i]w_i<0$ and $\sum_{i=0}^d h[i]w'_i>0$.
The samples in $\mathcal{S}$ discretize the search space.
If a hyperplane intersects a region such that the new regions are larger than $1/|\mathcal{S}|$, the intersection is expected to get identified using the samples.



In order to check if a new hyperplane $h$ intersects with a region $R$, we check for the existence of two points in $R$ such that
one belongs to $h^-$ and the other to $h^+$.
If so, we break $R$ in two new regions by partitioning the points inside $R$ between the two new regions.

Interestingly, it turns out that {\em despite the complexity of the arrangement and independent from the number of dimensions, a one-dimensional array can be used for partitioning the points.}
We use the 1D array $\ell$ for organizing the samples in a way that at any point of time, the samples that fall into a region are all beside each other, forming a range in the list.
For each region, we maintain the index of the first ($R.f$) and the last ($R.l$) samples falling into it.

Recall that the arrangement construction is iterative, and that every time a hyperplane $h$ is added to a region $R$, the algorithm partitions it into two sides: the intersection of $R$ with $h^-$ and its intersection with $h^+$.

Consider a set $\mathcal{S}$ of samples, drawn from \ar. Adding the first hyperplane $h_1$, partitions $\mathcal{S}$ to $\mathcal{S}\cap h_1^-$ and $\mathcal{S}\cap h_1^+$.
We can apply the famous {\em partition} algorithm, used in quick-sort~\cite{clrs}, and partition the samples in two sets such that all the samples in the range $[\mathcal{S}[1],\mathcal{S}[i]]$ belong to $R_1=\{h_1^-\}$ and all the ones with indices $i+1$ to $|\mathcal{S}|$ belong to$R_2=\{h_1^+\}$.
To do so, for every sample point $p$, we compute $s = \sum_{k=1}^d p.h_i[k]$. If $s<0$, $p$ belongs to $R_1$, otherwise $R_2$.

Interestingly, applying the partition algorithm on the points between $R.f$ and $R.l$ for a hyperplane $h$ both checks if $h$ intersects with $R$ and adds it to $R$ in the case of intersection. If the output index $i$ from the partition algorithm is $R.f-1$ or $R.l$, $h$ does not intersects $R$, otherwise the samples are already partitioned to $R_l= R\cap \{h_1^-\}$ and $R_r=R\cap \{h_1^+\}$, while $(R_l.sb,R_l.se)=(R.f,i)$ and $(R_r.sb,R_r    .se)=(i+1,R.l)$.


The partition algorithm on each subset $\ell_{R}$ of $\ell$ is in $O(|\ell_{R}|)$. since each region contains at least one point, and for adding a new hyperplane the algorithm partition is called only once, the total cost of adding a new hyperplane to the arrangement is $O(s)$, where $s$ is the number of samples in $\ell$. As a result, the total cost for constructing the arrangement of $\rho$ hyperplanes is (independent from the number of dimensions) $O(\rho\times s)$.

\section{Conclusion}
Data-driven decision making is often based on a single score for evaluating object and individuals. 
These scores can be obtained by combining different features either through a process learned by ML models, or using a weight vector designed by human experts, with their past experience and notions of what constitutes item quality.
While these scores provide an easy-to-understand representation, they can be erroneous, misleading, or have disparate impact, if not designed properly.

In this paper, focusing on linear scoring functions, we highlighted function sampling as a strong tool for achieving responsible scoring for tasks such as ranking and classification.
We proposed unbiased samplers for the entire function space, as well as a $\theta$-vicinity around a given function.
Without limiting the application scope of such a sampler for designing randomized algorithms, we showcased three diverse cases for fairness by design, characterizing the fitness of a dataset for fairness, and auditing cherry-picking, all in the context of ranking.
While, the demonstration has been on ranking, we would like to emphasize that similar approaches can be designed, using the proposed sampler, for other tasks such as classification. Finally, we showed how function sampling could be used to construct efficiently an approximate arrangement of hyperplanes, a higher level task that arises in many scenarios involving responsible scoring.

\bibliographystyle{unsrt}
\bibliography{ref}

\end{document}